\documentclass[a4paper,fleqn]{cas-sc}

\usepackage[numbers]{natbib}

\usepackage{algorithm}
\usepackage{algorithmic}
\usepackage{url}
\usepackage{color}

\usepackage{latexsym,mathrsfs}
\usepackage{amsmath,amssymb} 
\usepackage{amsthm,enumerate,verbatim}

\def\tsc#1{\csdef{#1}{\textsc{\lowercase{#1}}\xspace}}
\tsc{WGM}
\tsc{QE}
\tsc{EP}
\tsc{PMS}
\tsc{BEC}
\tsc{DE}

\newtheorem{lemma}{Lemma}

\DeclareMathOperator{\argmin}{argmin}

\DeclareMathOperator{\nnz}{nnz}

\begin{document}
\let\WriteBookmarks\relax
\def\floatpagepagefraction{1}
\def\textpagefraction{.001}
\shorttitle{Orthogonal NMF with the Kullback-Leibler divergence}
\shortauthors{Nkurunziza, Nahayo \& Gillis}

\title [mode = title]{Orthogonal Nonnegative Matrix Factorization with the Kullback-Leibler divergence}

\author[1]{Jean Pacifique Nkurunziza}[orcid=0009-0002-6595-6778]
\fnmark[1]
\ead{pacifiquenkuru997@gmail.com}

\author[1]{Fulgence Nahayo}[orcid=0000-0002-6540-7851]
\ead{fulgence.nahayo@ub.edu.bi}

\author[2]{Nicolas Gillis}[orcid=0000-0001-6423-6897]
\cormark[1] 
\ead{nicolas.gillis@umons.ac.be}
\ead[URL]{https://sites.google.com/site/nicolasgillis/}


\affiliation[1]{organization={LURMISTA/ISTA, University of Burundi},
                addressline={Boulevard de l'Indépendance 11}, 
                city={Bujumbura},
                postcode={6934}, 
                country={Burundi}}

\affiliation[2]{organization={University of Mons},
                addressline={Rue de Houdain 9}, 
                postcode={7000}, 
                postcodesep={}, 
                city={Mons},
                country={Belgium}}

\cortext[cor1]{Corresponding author}
\fntext[fn1]{JPS is associated to the National Doctoral School of Burundi.}

\begin{abstract}
Orthogonal nonnegative matrix factorization (ONMF) has become a standard approach for clustering. As far as we know, most works on ONMF rely on the Frobenius norm to assess the quality of the approximation. This paper presents a new model and algorithm for ONMF that minimizes the Kullback-Leibler (KL) divergence. As opposed to the Frobenius norm which assumes Gaussian noise, the KL divergence is the maximum likelihood estimator for Poisson-distributed data, which can model better sparse vectors of word counts in document data sets and photo counting processes in imaging. We develop an algorithm based on alternating optimization, KL-ONMF, and show that it performs favorably with the Frobenius-norm based ONMF for document classification and hyperspectral image unmixing. 
\end{abstract}











\maketitle

\section{Introduction}			

Given a data matrix $X \in \mathbb{R}^{m\times n}$ where each column, $X(:,j) \in \mathbb{R}^{m}$ for $j=1,2,\dots,n$, corresponds to a data points, and a factorization rank $r$, orthogonal nonnegative matrix factorization (ONMF)~\cite{ding2006orthogonal} requires to find $W\in \mathbb{R}^{m\times r}$ and $H \in \mathbb{R}^{r \times n}$ such that 
\[
X \approx WH, \quad H \geq 0, \quad \text{ and } HH^\top = I_r, 
\]
where $I_r$ is the $r$-b-$r$ identity matrix. 
The constraints on the matrix $H$ make ONMF a clustering problem~\cite{ding2006orthogonal}: if a matrix $H$ is component-wise nonnegative (that is, $H \geq 0$) and orthogonal (that is, $HH^\top = I_r$), then $H$ has at most a single non-zero  entry in each column (that is, in each $H(:,j)$). Hence each data points is approximated using a single column of $W$. In fact, let us denote $\mathcal K_k$ the set that contains the non-zero indices of the $k$th row of $H$, that is, $\mathcal K_k = \{ j \ | \ H_{k,j} \neq 0\}$. We have 
\[
X(:,j) \approx W(:,k) H_{k,j} \quad \text{ for all }  \quad j \in \mathcal K_k. 
\]
The data points are separated into $r$ clusters, $\{\mathcal K_k\}_{k=1}^r$, whose centroids are given by the columns of $W$. Standard $k$-means (and variants that use other distances than the squared Euclidean distance~\cite{banerjee2005clustering}) approximates each data point with the cluster centroid, that is, $X(:,j) \approx W(:,k)$ for $j \in \mathcal K_k$. 
ONMF approximates each data point as a \emph{multiple} of each centroid, and hence the angle between the centroids and the data points plays a more important role than their distances. In fact, ONMF was shown to be equivalent to a weighted variant of spherical $k$-means~\cite{pompili2014two}; see also~\cite[pp.~188-189]{gillis2020nonnegative}. 
Note that (i)~$X$ and $W$ are not required to be nonnegative, although this is often the case in practice, and (ii)~if $X$ is nonnegative, then the optimal $W$ also is. 

ONMF has attracted a lot of attention and has been shown to perform well for various clustering tasks; see~\cite{ding2006orthogonal, Choi08algorithmsfor, YooChoiOrthogonal2008, Yang10linearand, li2014two, pompili2014two, mirzal2014convergent, asteris2015orthogonal, wang2019clustering} and the references therein.
As far as we know, most previous works on ONMF focused on the Frobenius norm to assess the quality of the approximation, that is, they focused on the following optimization problem: given $X \in \mathbb{R}^{m\times n}$ and $r$,  solve 
\[
 \min_{W\in \mathbb{R}^{m\times r}, H \in \mathbb{R}^{r \times n}} 
 \, \| X - WH \|_F^2 
  \text{ s.t. }   H \geq 0 \text{ and } HH^\top = I_r, 
\]
where $\|X\|_F^2 = \sum_{i,j} X_{i,j}^2$ is the squared Frobenius norm. The only exceptions we were able to find are the papers~\cite{li2007non, kimura2016column} in which the authors consider the KL and Bregman divergences but they used a regularization term to promote orthogonality, leading to soft clusterings; in this paper we impose   orthogonality as a hard constraint.   
The underlying assumption when using the Frobenius norm is that the noise follows a Gaussian distribution. 
However, in situations when Gaussian noise is not meaningful, other objective functions should be used. In particular, and this will be the focus of this paper, if the entries of the data matrix follow a Poisson distribution of parameter given by $WH$, then one should minimize the Kullback-Leibler (KL) divergence between $X$ and $WH$, 
which is defined as   
\[
D_{\text{KL}}(X,WH) = \sum_{i,j} D_{\text{KL}}(X_{i,j},(WH)_{i,j}), 
\]
where $D_{\text{KL}}(x,y) = y - x + x \log \frac{x}{y}$, 
and, by convention, $D_{\text{KL}}(0,y) = y$, while the KL divergence is not defined for $x < 0$. 
The KL divergence is for example much more meaningful for sparse document data sets where each column of $X$ is a vector of word count~\cite{lee1999learning}, and in some imaging applications~\cite{richardson1972bayesian, lucy1974iterative}.

\paragraph{Outline and contribution of the paper} 

In this paper, we propose an algorithm for 
ONMF with the KL divergence, using alternating optimization. As for the Frobenius norm which is described in Section~\ref{sec:FroONMF}), we will be able to derive closed-form updates, and hence devise a simple, yet effective, algorithm for ONMF with the KL divergence, which we will refer to as KL-ONMF; see Section~\ref{sec:KLonmf}. We will show that it compares favorably with Fro-NMF for document classification and hyperspectral image unmixing; see Section~\ref{sec:numexp}.

\paragraph{Notation} The matrix $X \in \mathbb{R}^{m\times n}$ is a real  $m$-by-$n$ matrix, its $j$th column is denoted $X(:,j)$, its $i$th row by $X(i,:)$, its entry at position $(i,j)$ by $X_{i,j}$, its transpose by $X^\top$, and its number of non-zero entries by $\nnz(X)$. 
Given sets of indices $\mathcal{I}$ and $\mathcal{J}$, $X(\mathcal{I},\mathcal{J})$ denotes the submatrix of $X$ with  row (resp.\ column) indices in $\mathcal{I}$ (resp.\ $\mathcal{J}$). The vectors $e$ is the vector of all ones of appropriate dimension, and $I_r$ is the identity matrix of dimension $r$.

\section{Alternating optimization for 
ONMF with the Frobenius norm} \label{sec:FroONMF}

Many algorithms have been developed for ONMF. Let us recall a simple and effective one, based on alternating optimization proposed in~\cite{pompili2014two}, which we refer to as Fro-ONMF.  It follows the same scheme as other clustering algorithms, in particular $k$-means that alternatively updates the centroids and the clusters; see also, e.g., \cite{banerjee2005clustering} for generalizations to any Bregman divergences. 
Fro-ONMF 
 updates $W$ and $H$ alternatively with closed-form expressions which can be derived as follows.  

\noindent $\bullet$ For $H$ fixed, since $W$ is unconstrained, the optimal solution must have its gradient equal to zero: 
  \[
2 (WH-X)H^\top = 0 \quad \Rightarrow \quad W =  XH^\top, 
  \]
  since $HH^\top = I_r$. It is interesting to interpret this closed-form expression. 
Since $H$ is orthogonal, there is a single non-zero entry per column, and recall $\mathcal K_k$ is the set that contains the non-zero indices of the $k$th row of $H$. 
We have 
\begin{equation} \label{eq:updateWfro}
W(:,k) 
=
X(:,\mathcal K_k) H(k,\mathcal K_k)^\top 
= 
\sum_{j \in \mathcal K_k} X(:,j) H_{k,j}. 
\end{equation} 
This means that $W(:,k)$ is a \emph{weighted} average of the data points (that is, the columns of $X$) belonging to the $k$th cluster corresponding to the $k$th row of $H$. 
The weight for each data point will depend on its norm since we will see that $H_{k,j}$ is equal to $\frac{X(:,j)^\top W(:,k)}{\|W(:,k)\|_2^2}$.  

\noindent $\bullet$  For $W$ fixed, 
  assume we know the position where $H(:,j)$ is different from zero, say $H_{k,j} \neq 0$. We must minimize $  \| X(:,j) - W(:,k) H_{k,j}\|_2^2 $ which is equal to 
  \[ 
   \| X(:,j) \|_2^2  - 2 X(:,j)^\top W(:,k) H_{k,j} + H_{k,j}^2 \|W(:,k)\|_2^2. 
  \] 
  If $X(:,j)^\top W(:,k) \geq 0$ (this will always hold when \mbox{$X \geq 0$} and $W \geq 0$), 
the optimal $H_{k,j}^* = \frac{X(:,j)^\top W(:,k)}{\|W(:,k)\|_2^2}$, otherwise, $H_{k,j}^* = 0$.  
Hence  $\| X(:,j) - W(:,k) H^*_{k,j}\|_2^2$  is equal to 
 \mbox{$\| X(:,j) \|_2^2 - \left(\frac{X(:,j)^\top W(:,k)}{\|W(:,k)\|_2}\right)^2$}. 
  Therefore, the non-zero entry of $H(:,j)$ will be the entry $k$ that maximizes  $\frac{W(:,k)^\top X(:,j)}{\|W(:,k)\|_2}$. 
  Hence, to update $H$, we first normalize $W$ so that its columns have unit $\ell_2$ norm, then   
  compute $W^\top X \in \mathbb{R}^{r \times n}$ and the maximum in each column will correspond to the non-zero entry in each column of $H$. Finally we update these entries with the above closed-form expression.  
  
  Note that, after this update, the rows of $H$ might not have norm 1; which is required by the constraint $H H^\top = I_r$. This can be fixed, without changing the objective function value, by rescaling the solution $WH$ since there is a scaling degree of freedom: for any $\alpha > 0$, 
  $W(:,k)H(k,:)$ $=$ \mbox{$(\alpha W(:,k))(H(k,:)/\alpha)$}. 
  If $W$ is updated after $H$, the scaling of $W$ is not necessary since it will be automatically scaled with the optimal closed-form expression provided in~\eqref{eq:updateWfro}.

Algorithm~\ref{alg:froONMF} summarizes alternating optimization for ONMF with the Frobenius norm, which we refer to as Fro-ONMF. 
Note that $X$ and $W$ do not need to be nonnegative, ONMF in the Frobenius norm can be used to cluster data points with negative entries. 

\begin{algorithm}[H]
	\caption{Fro-ONMF - alternating optimization for ONMF with the Frobenius norm~\cite{pompili2014two}} \label{alg:froONMF} 
	\begin{algorithmic}[1]
		\REQUIRE{Data matrix: $X\in \mathbb{R}^{m\times n}$, 
  Initialization: $W \in \mathbb{R}^{m \times r}$, 
  factorization rank: $r \ll n$, maximum number of iterations: maxiter, convergence criterion: $\delta \ll 1$.}
  
		\ENSURE{Matrices $W\in \mathbb{R}^{m\times r}$ and $H\in \mathbb{R}^{r\times n}_{+}$ with \mbox{$HH^\top = I_r$} such that   $\|X-WH\|_F^2$ is minimized.}

		\STATE $t=1$, $H = 1$, $H^{(p)} = 0$. 

		\WHILE {$t \leq$ maxiter and $\|H-H^{(p)}\|_F \geq \delta$} 
             
             \STATE  \emph{\% Update $H$ for $W$ fixed} 

        \STATE   $H^{(p)} \leftarrow H$. \emph{\% $H^{(p)}$ is the previous iterate to monitor convergence.}

\STATE Normalized $W$: $W_n(:,k) = W(:,k)/\|W(:,k)\|_2$  for all $k$. 
        
         \STATE Compute $A = W_n^\top X \in \mathbb{R}^{r \times n}$.  
         
         \STATE Let 
         $\mathcal{K}_k = \{ j \ | \ A(k,j) > A(k',j) \text{ for all } k' \neq k\}$. 

		\STATE Update $H(k,j) \leftarrow  \frac{W(:,k)^\top X(:,j)}{\|W(:,k)\|_2^2}$ for all $k$ and $j \in \mathcal{K}_k$.

   \STATE Scale $H$: $H(k,:) \leftarrow \frac{H(k,:)}{\|H(k,:)\|_2}$ for all $k$. 
   
		 \STATE  \emph{\% Update $W$ for $H$ fixed} 
 
		\STATE $W(:,k) \leftarrow X(:,\mathcal K_k)  H(k,\mathcal K_k)^\top$ for all $k$.  

            \STATE $t \leftarrow t+1$
		\ENDWHILE 
	\end{algorithmic}
\end{algorithm}

Let us discuss three important aspects of Algorithm~\ref{alg:froONMF}. 

\paragraph{Computational cost} The main cost of Algorithm~\ref{alg:froONMF} is the matrix-matrix product $W_n^\top X$ which requires $O(\nnz(X)r)$ operations, where $\nnz(X)$ is the number of non-zero entries of $X$ (it is equal to $mn$ if $X$ is dense). 
The other operations requires $O(\nnz(X))$ operations or less, assuming $\nnz(X) \gg \max(m,n) r$ (which should be the case otherwise the number of non-zero entries of $X$ is smaller than the number of parameters in ONMF). This makes Algorithm~\ref{alg:froONMF} very fast and scalable.  

Note that the update of $W$ can be computed directly as $W \leftarrow XH^\top$. However, if using a dense matrix to represent $H$, this would require $O(\nnz(X)r)$ operations, which would make this step $r$ times slower\footnote{In fact, we have improved the implementation of Fro-ONMF that used to update $W$ using $XH^\top$ which required $O(\nnz(X)r)$ operations, hence making the cost per iteration almost double.}. 

\paragraph{Convergence} By construction, the objective function, $\|X - WH\|_F^2$, which is bounded below, decreases at each step, and hence objective function values converge. 
Since the feasible set for $H$ is compact, and the level sets are compact (because $\|X - WH\|_F^2$ is coercive in $W$ as $H$ is normalized with non-zero rows, meaning that $\|X - WH\|_F^2$ goes to infinity as any entry of $W$ goes to infinity), there exists a converging subsequence of the iterates.

\paragraph{Stopping criterion} Since $H$ is normalized ($\|H(k,:)\|_2 = 1$ for all $k$), it makes sense to use the (cheap) stopping criterion $||H-H^{(p)}||_F < \delta$, where $\delta$ is a parameter smaller than 1. 
We will use $\delta = 10^{-6}$. We will also use a maximum of 100 iterations (which is never reached in our numerical experiments).

Another important aspect of ONMF algorithms is the initialization of $W$ which will be discussed in Section~\ref{sec:numexp}.

\section{Alternating optimization for 
ONMF with the KL divergence} \label{sec:KLonmf}

As explained in the introduction, we focus in this paper on ONMF with the KL divergence, which can be formulated as follows: given $X \in \mathbb{R}^{m\times n}$ and $r$,  
solve  
\begin{equation} \label{eq:KL-ONMF}
 \min_{W\in \mathbb{R}^{m\times r}, H \in \mathbb{R}^{r \times n}} 
 \, D_{\text{KL}} (X , WH ) 
  \text{ s.t. }   H \geq 0 \text{ and } HH^\top = I_r. 
\end{equation}
Note that, as opposed to ONMF with the Frobenius norm, the KL divergence requires $X$ to be component-wise nonnegative, because the KL divergence is defined only for $X \geq 0$. 

In the next two sections, we detail the closed-form expressions for $W$ when $H$ is fixed in~\eqref{eq:KL-ONMF}, and vice versa.

  \subsection{Update $W$ when $H$ is fixed}

 Since $H$ is orthogonal and non-negative, each column of $W$ needs only to be optimized over a subset of the columns of $X$, as for Fro-ONMF. 
 Recall the notation $\mathcal{K}_k = \{ j | H_{k,j} > 0\}$. 
 To find the optimal $W(:,k)$, we need to solve the following problem: 
	\begin{equation}\label{eq24_2_1}
		\min\limits_{W(:,k)} \sum\limits_{j \in \mathcal{K}_k} D_{\text{KL}}(X(:,j), W(:,k) H_{k,j} ).  
	\end{equation} 
\begin{lemma}\label{lem2}
    The optimal solution of the problem \eqref{eq24_2_1} is given by 
\[ 
W^*(:,k) = \frac{X(:,\mathcal K_k) e}{H(k,:) e}. 
\] 
\end{lemma} 
\begin{proof} 
The derivative w.r.t.\ $W(:,k)$ of $D_{\text{KL}}(X,WH)$ is given by~\cite{lee1999learning} 
\[
e e^\top H(k,:)^\top - \frac{[X]}{[WH]} H(k,:)^\top. 
\]
Using the structure of \eqref{eq24_2_1} to simplify the above expression, and setting the gradient to zero for
 optimality (as we will see, the unconstrained solution is nonnegative as long at $X$ and $H$ are), we obtain  that $e e^\top  H(k,:)^\top $ is equal to 
\begin{align*}
e e^\top  H(k,\mathcal K_k)^\top 
& = 
\frac{[X(:,\mathcal K_k)]}{[W(:,k)H(k,\mathcal K_k)]}  H(k,\mathcal K_k)^\top  \\ 
& = 
\frac{[X(:,\mathcal K_k) e]}{[W(:,k)]} . 
\end{align*} 
Since $e^\top  H(k,\mathcal K_k)^\top = H(k,\mathcal K_k) e$,  this gives 
$W^*(:,k) = \frac{X(:,\mathcal K_k) e}{H(k,\mathcal K_k) e}$ at optimality.  
\end{proof}

It is interesting to note that the optimal $W(:,k)$ is a scaled average of the columns of $X$ that are associated with the cluster $k$ defined by the $k$th row of $H$. This is rather different than Fro-ONMF where $W(:,k)$ is a weighted average where the weights depend linearly on the norm of the data points;  see the discussion around Equation~\eqref{eq:updateWfro} in Section~\ref{sec:FroONMF}. This implies that the centroids depend \emph{quadratically} on the norms of the data points. 
Hence KL-ONMF will be less sensitive to outliers, and will give more importance to data points with smaller norm, as opposed to Fro-ONMF. The reason is that the KL divergence KL$(x,y)$ grows linearly with $y$ for $y$ sufficiently large (as opposed to quadratically for the Frobenius norm).  
This will be illustrated in the numerical experiments; in particular for hyperspectral images where KL-ONMF will be able to identify materials with small spectral signatures, while Fro-ONMF will not be able to do so.

  \subsection{Update $H$ when $W$ is fixed}

As for the Frobenius norm, let us first assume we know the position at which $H(:,j)$ is non-zero, we will then select the entry that minimizes $D_{\text{KL}} (X(:,j), W(:,k) H_{k,j} )$ the most. 
 \begin{lemma}\label{lem1} Given $X(:,j) \geq 0$ and $W(:,k) \geq 0$, we have 
\[ 
H_{k,j}^* =  
\argmin_{H_{k,j}} D_{\text{KL}} (X(:,j), W(:,k) H_{k,j} ) 
 = \frac{e^\top X(:,j)}{e^\top W(:,k)},  
\] 
where $e$ is the vector of all ones. 
\end{lemma}
\begin{proof}
Let us consider the problem 
$\min_{\alpha} D_{\text{KL}}(X(:,j) , \alpha W(:,k) )$,  
where $\alpha$ represents $H_{k,j}$. 
We need to minimize \vspace{-0.2cm}  
\begin{align*}
    f(\alpha) 
    = D_{\text{KL}}(X(:,j) , \alpha W(:,k) ) 
& = \alpha e^\top W(:,k)  
- \sum_{i=1}^m X(i,j) \log(\alpha W(i,k)) \vspace{-0.2cm}  \\ 
& = \alpha e^\top W(:,k)  
- e^\top X(:,j) \log \alpha  -  \sum_{i=1}^m X(i,j) \log W(i,k). \vspace{-0.2cm}  
\end{align*} 
The derivative with respect to $\alpha$ is given by 
$f'(\alpha) = e^\top W(:,k)  
- \frac{1}{\alpha} e^\top X(:,j)$. 
Hence the optimal solution is given by  
$H_{k,j}^{\ast} = \alpha^{\ast} 
= \frac{e^\top X(:,j)}{e^\top W(:,k)}$. 
\end{proof}
 This means that $H_{k,j}$ is the ratio between the average entries in the corresponding columns of $X$ and $W$.

\noindent Then, to decide which column of $W$ is best for $X(:,j)$, let us compute the error we would get for each choice of $H_{k,j}^{\ast}$: \vspace{-0.2cm}   
\begin{align*}
    KL(X(:,j), H_{k,j}^* W(:,k)) & 
= 
\frac{e^\top X(:,j)}{e^\top W(:,k)} e^\top W(:,k)
- \sum_{i=1}^m X(i,j) \log\left(\frac{e^\top X(:,j)}{e^\top W(:,k)} W(i,k)\right) \\ 
& = e^\top X(:,j)
- \sum_{i=1}^m X(i,j) \log\left(\frac{  W(i,k) }{e^\top W(:,k)}\right)
- \sum_{i=1}^m X(i,j) \log ( e^\top X(:,j) ).  
\end{align*} 

\noindent The entry of the column of $H(:,j)$ that should be set to a positive value is the one that  minimizes the quantity \mbox{$KL(X(:,j), H_{k,j}^* W(:,k))$}, hence it is the entry corresponding to the index $k$ such that $\sum\limits_{i=1}^m X(i,j) \log\left(\frac{  W(i,k) }{e^\top W(:,k)}\right)$ is maximized. That is, to select the column of $W$ that best approximates  the column $X(:,j)$ in the KL divergence, we pick the column $W(:,k)$ maximizing\footnote{Recall that, by convention, $0 \log 0 = 0$, as KL$(0,y) = y$.} $\sum\limits_{i=1}^m X(i,j) \log\left(\frac{  W(i,k) }{e^\top W(:,k)}\right)$. 
Numerically, we can construct $W_n$ by  normalizing the columns of $W$ to have unit $\ell_1$, and then compute $A:= \log(W_n + \epsilon)^\top X$, where $\epsilon$ is a small constant to avoid numerical errors, and the $\log$ is taken component-wise: the locations of the non-zero entries of the optimal $H$ are given by the maximum of each column of $A$.

\subsection{Algorithm KL-ONMF} \label{sec:algoKLonmf}

Algorithm~\ref{alg:klONMF} summarizes the alternating optimization scheme, which we refer to as KL-ONMF.

\begin{algorithm}[H]
	\caption{KL-ONMF - alternating optimization for ONMF with the KL divergence} \label{alg:klONMF} 
	\begin{algorithmic}[1]
		\REQUIRE{Nonnegative matrix: $X\in \mathbb{R}^{m\times n}_{+}$, 
  Initialization: $W \in \mathbb{R}^{m \times r}_+$, 
  factorization rank: $r$, maximum number of iterations: maxiter, convergence criterion: $\delta \ll 1$, small parameter: $\epsilon \ll 1$.}
  
		\ENSURE{Matrices $W\in \mathbb{R}^{m\times r}_{+}$ and $H\in \mathbb{R}^{r\times n}_{+}$ with $HH^\top = I_r$ such that   $D_{\text{KL}}(X,WH)$ is minimized.}

		\STATE $t=1$, $H = 1$, $H^{(p)} = 0$. 

		\WHILE {$t \leq$ maxiter and $\|H-H^{(p)}\|_F \geq \delta$} 
             
             \STATE  \emph{\% Update $H$ for $W$ fixed} 

        \STATE   $H^{(p)} \leftarrow H$. \emph{\% $H^{(p)}$ is the previous iterate to monitor convergence.}

\STATE Normalized $W_n$: $W(:,k) \leftarrow W(:,k)/e^\top W(:,k)$  for all $k$.  
        
         \STATE Compute $A = \log(W_n + \epsilon)^\top X \in \mathbb{R}^{r \times n}$. 
         
         \STATE Let 
         $\mathcal{K}_k = \{ j \ | \ A(k,j) > A(k',j) \text{ for all } k' \neq k\}$. 

		\STATE Update $H(k,j) \leftarrow  \frac{e^\top X(:,j)}{e^\top W(:,k)}$ for all $k$ and $j \in \mathcal{K}_k$.

   \STATE Scale $H$: $H(k,:) \leftarrow \frac{H(k,:)}{\|H(k,:)\|_2}$ for all $k$. 
   
		 \STATE  \emph{\% Update $W$ for $H$ fixed} 
 
		\STATE $W(:,k) \leftarrow \frac{X(:,\mathcal K_k) e}{H(k,:) e}$ for all $k$.

            \STATE $t \leftarrow t+1$
		\ENDWHILE  
	\end{algorithmic}
\end{algorithm}

Let us discuss some details about Algorithm~\ref{alg:klONMF}. 

\paragraph{Computational cost} 
As for Fro-ONMF (Algorithm~\ref{alg:froONMF}), the main cost is a matrix-matrix product, namely computing $\log(W + \epsilon)^\top X$, which requires $O(\nnz(X)r)$ operations. The other operations requires $O(\nnz(X))$ operations or less, assuming $\nnz(X) \gg \max(m,n) r$ (which should be the case otherwise the number of non-zero entries of $X$ is smaller than the number of parameters in ONMF). 
This makes Algorithm~\ref{alg:klONMF} have essentially the same computational cost as Fro-ONMF, hence being very fast and scalable. 
We have observed in practice that KL-ONMF is slightly faster than Fro-ONMF, the reason is the update of $W$: KL-ONMF only needs to perform $\nnz(X) + nr$ sums to compute $W(:,k) = X(:,\mathcal K_k)e/H(k,:)e$, 
while Fro-ONMF requires $\nnz(X)$ sums and products for the update of $W(:,k)$ $=$ \mbox{$X(:,\mathcal K_k) H(k,K_k)^\top$}. 

\paragraph{Convergence} The same observations as for Fro-ONMF apply: the objective function values will converge while there is a converging subsequence of iterates.

\paragraph{Stopping criterion} As for Fro-ONMF, we use the stopping criterion $||H-H^{(p)}||_F < \delta \ll 1$, and we will use $\delta = 10^{-6}$. We will also use a maximum of 100 iterations, which is never reached in our numerical experiments. The parameter $\epsilon$ that allows to take the $\log$ of $W$ is set to $10^{-3}$.

\section{Numerical Experiments} \label{sec:numexp}

In this section, we compare the performance of Fro-ONMF (Algorithm~\ref{alg:froONMF})  and KL-ONMF (Algorithm~\ref{alg:klONMF}) for clustering documents (Section~\ref{sec:docu}) and pixels in hyperspectral images (Section~\ref{sec:hsi}). 
The code in MATLAB is available from \url{https://gitlab.com/ngillis/kl-onmf}, and can be used to rerun all experiments presented below. All experiments were run on a LAPTOP Intel(R) Core(TM) i7-8850H CPU @ 2.60GHz 16,0Go RAM.

\paragraph{Initialization} There are many ways to initialize ONMF algorithms, as they are for $k$-means. 
To simplify the presentation, we use the approach proposed in~\cite{gillis2020nonnegative}.  It initializes $W$ with the successive nonnegative projection algorithm (SNPA)~\cite{gillis2014successive} that identifies a subset of $r$ columns  of $X$ that represent well-spread data points in the data set.

\subsection{Document data sets} \label{sec:docu}

We first use ONMF to cluster the 15 document data sets from~\cite{ZG05}. Table~\ref{table:doc} reports the name of the data sets as well as their dimensions ($m$ is the number of words, $n$ the number of documents) and the number of clusters ($r$). 
For each data set, Table~\ref{table:doc} reports the accuracy in percent of the clustering obtained by Fro-ONMF and KL-ONMF, denoted acc-F and acc-KL, respectively. 
Given the true disjoint clusters 
$C_i \subset \{1,2,\dots,n\}$ for $1 \leq i \leq r$ and given a computed disjoint clustering $\{\tilde{C}_i\}_{i=1}^r$, 
its {\em accuracy} is defined as 
\[
\text{accuracy}\left( \{\tilde{C}_i\}_{i=1}^r \right)
\; = \;  
\min_{\pi \in [1,2,\dots,r]} \frac{1}{n} \sum_{i=1}^r |C_i \cap \tilde{C}_{\pi(i)}|,  
\]
where $[1,2,\dots,r]$ is the set of permutations of $\{1,2,\dots,r\}$. The last row of Table~\ref{table:doc} reports the weigthed average accuracy, where the weight is the number of documents to be clustered. 
Table~\ref{table:doc} also reports the average run times in seconds over 20 runs, denoted time-F and time-KL, respectively, 
as well as the number of iterations needed for the two algorithms to converge, denoted it-F and it-KL, respectively. 
For comparison, we also report the best accuracy for 6 ONMF algorithms reported in~\cite{pompili2014two} which we denote acc-best (unfortunately, only for 12 out of the 15 data sets). 

\begin{center} 
\begin{table*}[h!] 
\caption{Fro-ONMF vs.\ KL-ONMF for the clustering of 15  document data sets: 
$m$ is the number of words, 
$n$ is the number of documents, 
$r$ is the number of clusters, 
acc-F and acc-KL are the accuracies in percent for Fro-ONMF and KL-ONMF, resp., 
acc-best is the best results among 6 ONMF algorithms reported in~\cite{pompili2014two},  
time-F and time-KL are the run times in seconds for Fro-ONMF and KL-ONMF, resp., and 
it-F and it-KL are the number of iterations needed to converge for Fro-ONMF and KL-ONMF, resp. The best result between Fro-ONMF and KL-ONMF is highlighted in bold. 
\label{table:doc}} 
\begin{center}  
\begin{tabular}{|c||c|c|c||c|c||c||c|c||c|c|} 
\hline 
& $n$ & $m$ & $r$ & acc-F & acc-KL & acc-best & time-F & time-KL & it-F & it-KL \\ \hline 
\hline 
NG20 & 19949 & 43586 & 20  & 26.5 & \textbf{48.8} & NA & 5.5 & \textbf{2.3}  & 38 & \textbf{34}\\ \hline 
ng3sim & 2998 & 15810 & 3  & 38.0 & \textbf{71.4} &  NA  & 0.3 & \textbf{0.1}  & 24 & \textbf{11}\\ \hline 
classic & 7094 & 41681 & 4  & 55.9 & \textbf{85.4} &  58.8 & 1.1 & \textbf{0.3}  & 97 & \textbf{37}\\ \hline 
ohscal & 11162 & 11465 & 10  & 29.1 & \textbf{47.7} & 39.2  & 1.3 & \textbf{0.6}  & \textbf{31} & 36\\ \hline 
k1b & 2340 & 21819 & 6  & \textbf{74.4} & 58.5 & 79.0  & 0.5 & \textbf{0.1}  & 40 & \textbf{10}\\ \hline 
hitech & 2301 & 10080 & 6  & \textbf{39.6} & 38.5 &  48.7  & 0.4 & \textbf{0.1}  & 28 & \textbf{18}\\ \hline 
reviews & 4069 & 18483 & 5  & 51.4 & \textbf{72.4} & 63.7  & 0.4 & \textbf{0.2}  & 14 & \textbf{12}\\ \hline 
sports & 8580 & 14870 & 7  & 48.9 & \textbf{66.6} &  50.0 & 1.2 & \textbf{0.6}  & \textbf{24} & 29\\ \hline 
la1 & 3204 & 31472 & 6  & 56.2 & \textbf{61.2} & 65.8  & 0.7 & \textbf{0.3}  & 28 & \textbf{25}\\ \hline 
la12 & 6279 & 31472 & 6  & 45.7 & \textbf{67.5} &  NA  & 2.7 & \textbf{0.7}  & 59 & \textbf{29}\\ \hline 
la2 & 3075 & 31472 & 6  & \textbf{62.2} & 59.9 &  52.8 & 0.6 & \textbf{0.2}  & 27 & \textbf{13}\\ \hline 
tr11 & 414 & 6424 & 9  & 50.5 & \textbf{54.1} & 50.2  & 0.1 & \textbf{0.0}  & 19 & \textbf{ 9}\\ \hline 
tr23 & 204 & 5831 & 6  & \textbf{43.1} & 34.3 &  41.2 & {0.0} & 0.0  & \textbf{ 8} & 16\\ \hline 
tr41 & 878 & 7453 & 10  & 44.2 & \textbf{48.6} & 53.2  & 0.2 & \textbf{0.1}  & 25 & \textbf{15}\\ \hline 
tr45 & 690 & 8261 & 10  & 42.2 & \textbf{59.6} & 41.4 & 0.1 & \textbf{0.1}  & 13 & \textbf{10}\\ \hline 
\hline 
Averages &  &  &  & 41.2 & \textbf{59.2} & 54.0 & 1.0 & \textbf{0.4} & 32 & \textbf{20} \\ \hline 
\end{tabular} 
\end{center}
\end{table*}
\end{center} 

We observe the following: 
\begin{itemize}
    \item On average, KL-ONMF provides significantly better clustering, with 59,2\% accuracy compared to 41,2\% for Fro-NMF. However, for some data sets, Fro-NMF sometimes provides significantly better solutions (for k1b and tr23).

    \item Compared to the 6 other ONMF algorithms (which used different initializations), KL-ONMF still performs best on average. 

    \item KL-ONMF requires on average less iterations than Fro-ONMF to converge while its cost per iteration is slightly cheaper (see the discussion in Section~\ref{sec:algoKLonmf}), leading to faster computational times, on average more than 2 times faster. 
    
\end{itemize}

\subsection{Hyperspectral images} \label{sec:hsi}

A hyperspectral image (HSI) can be represented as an $m$-by-$n$ matrix where $m$ is the number of wavelength measured which is typically between 100 and 200, as opposed to 3 bands for color images (red, green and blue), and 
$n$ is the number of pixels. When most pixels in a HSI contain a single material, referred to as an endmember in the HSI literature, it makes sense to cluster them according to the endmember they contain; see, e.g.,~\cite{gillis2014hierarchical} and the references therein. 

Let us compare Fro-ONMF and KL-ONMF on three widely used data sets: Moffet, Samson and Jasper. For these data sets, we have good estimates of the ground-truth endmembers; see~\cite{zhu2017hyperspectral} and the references therein.   
Hence we can evaluate the quality of the endmembers extracted by an ONMF algorithm by comparing it to the ground truth. 
The standard metric to do so is the mean removed spectral angle (MRSA) between two vectors $x \in \mathbb{R}^{m}$ and $y \in \mathbb{R}^{m}$: 
	\[ 
  \text{MRSA}(x,y)=\frac{100}{\pi}\cos^{-1}\left(\frac{(x-\bar{x}e)^\top (y-\bar{y}e)}{\|x-\bar{x}e\|_2 \|y-\bar{y}e\|_2}\right), \]
	where $\bar{x}=\frac{1}{n}\sum_{i=1}^n x(i)$. Note that $\text{MRSA}(x,y) \in [0,100]$,  and the smaller the MRSA, the better approximation $y$ is of $x$.  
 Table~\ref{table:HSImrsa} reports the average MRSA between the columns of the estimated $W$'s  and the ground-truth $W_t$, using an optimal permutation.  
\begin{center} 
\begin{table*}[h!] 
\caption{Fro-ONMF vs.\ KL-ONMF for the clustering of 3  HSIs: 
$m$ is the number of spectral bands, 
$n$ is the number of pixels, 
$r$ is the number of endmembers, 
MRSA-F and MRSA-KL are the average MRSA compared to the ground truth for Fro-ONMF and KL-ONMF, resp., 
time-F and time-KL are the run times in seconds for Fro-ONMF and KL-ONMF, resp., and 
it-F and it-KL are the number of iterations needed to converge for Fro-ONMF and KL-ONMF, resp.  The best result is highlighted in bold. \label{table:HSImrsa}} 
\begin{center}  
\begin{tabular}{|c||c|c|c||c|c||c|c||c|c|} 
\hline  
& $m$ & $n$ & $r$ & MRSA-F & MRSA-KL & time-F & time-KL & it-F & it-KL \\ \hline 
Moffet & 159 & 2500 & 3 & 26.1 
& \textbf{7.77} & \textbf{0.09} & 0.11 & 10 & {10} \\ 
Samson & 156 & 9025 & 3 & \textbf{2.74} & 2.75& 0.22 & 0.22 & 7 & \textbf{6} \\  
Jasper & 198 & 10000 & 4 & 19.6 & \textbf{3.6} & \textbf{0.49} & 0.50  & 22 & \textbf{16} \\ \hline 
\end{tabular} 
\end{center}
\end{table*}
\end{center} 

\newpage 
We observe the following: 
\begin{itemize}
    \item For the Samson data set, Fro-ONMF and KL-ONMF provide very similar results. 

    \item For the Moffet and Jasper data sets, KL-ONMF outperforms Fro-ONMF, with significantly smaller MRSA. The reason is that these two data sets contains an endmember with small norm (that is, a column of $W$ with small norm) corresponding to the water; see Figure~\ref{fig:spectralsign}. 
    Since Fro-NMF tends to favor endmembers with large norms, it is unable to extract the water properly in these two data sets, and hence has a significantly larger MRSA than KL-ONMF. 
    
    Figures~\ref{fig:spectralsignMoffet} and~\ref{fig:spectralsignJasper} display  the clustering of the pixels for the three HSIs for Fro-ONMF and KL-ONMF. We observe in fact that that KL-ONMF perfectly extracts the water on the Moffet and Jasper data sets, while Fro-ONMF completely fails to do so. 


    \item In terms of runtime and number of iterations, Fro-ONMF and KL-ONMF have similar performances. 
\end{itemize}

\begin{figure*}[ht!]
\begin{center}
\includegraphics[width=\textwidth]{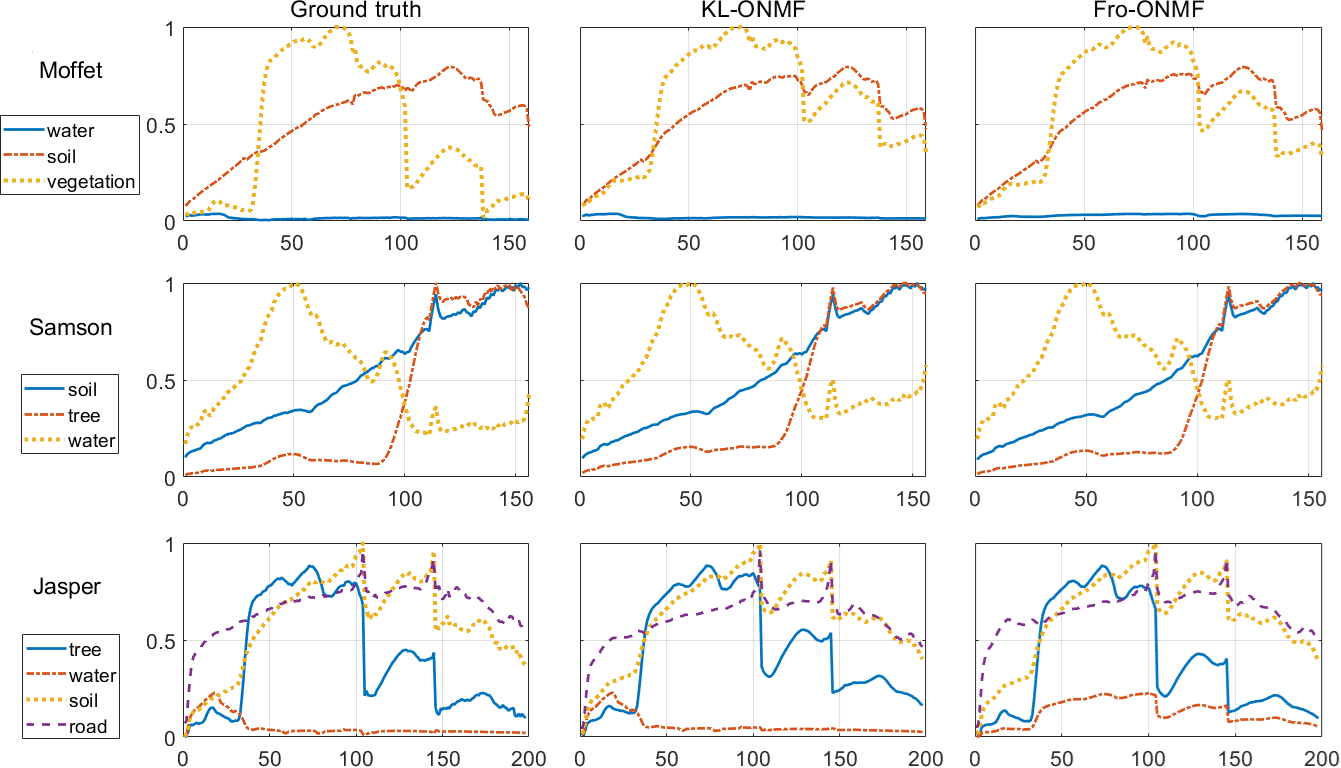} 
\caption{Spectral signatures extracted with KL-ONMF and Fro-ONMF on hyperspectral images (these corresponds to the columns of the computed $W$ factor); from top to bottom: Moffet, Samson and Jasper.} 
\label{fig:spectralsign}
\end{center}
\end{figure*}

\begin{figure*}[ht!]
\begin{center}
\begin{tabular}{ccc}
\textbf{Moffet} & & \textbf{Samson} \vspace{0.2cm} \\ 
 \includegraphics[width=0.4\textwidth]{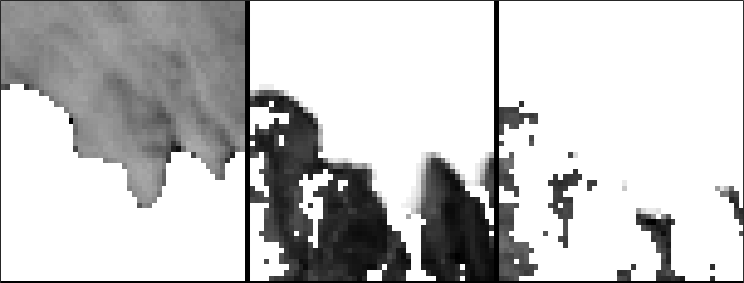} &  & \includegraphics[width=0.4\textwidth]{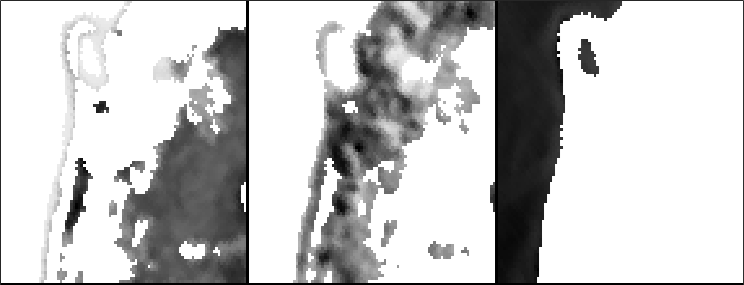}  \\
 Ground truth & & Ground truth    \vspace{0.2cm}    \\ 
 \includegraphics[width=0.4\textwidth]{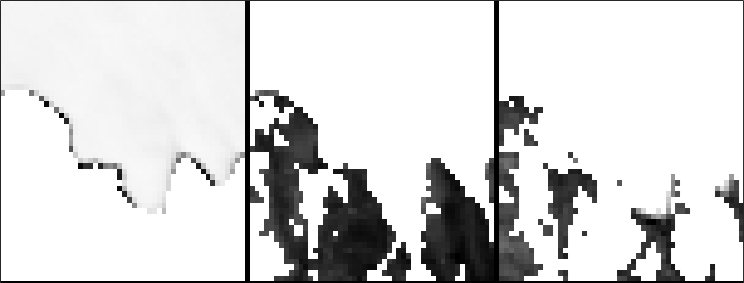} & &  
  \includegraphics[width=0.4\textwidth]{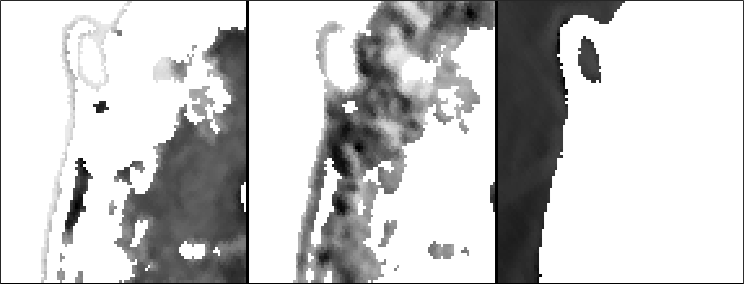} \vspace{-0.1cm} \\ 
  Fro-ONMF & & Fro-ONMF  \vspace{0.2cm}  \\
   \includegraphics[width=0.4\textwidth]{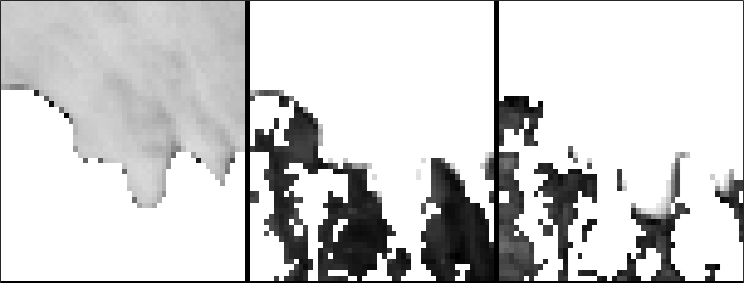}   & & 
   \includegraphics[width=0.4\textwidth]{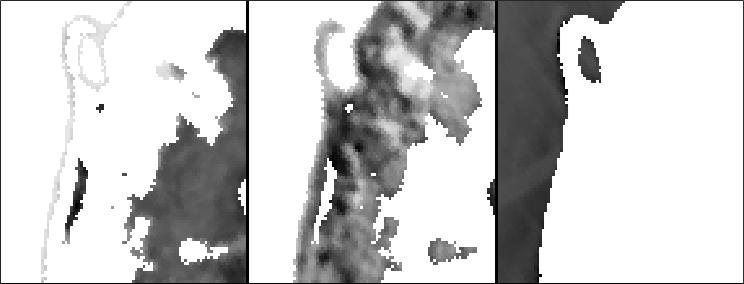}  \\ 
   KL-ONMF & & KL-ONMF 
\end{tabular}
\caption{Clustering of the Moffet HSI (from left to right: water,  soil and vegetation), and the Samson HSI (from left to right: soil, tree, water).  \label{fig:spectralsignMoffet}}  
\end{center}
\end{figure*}

\begin{figure}[ht!]
\begin{center}
\begin{tabular}{c}
 \includegraphics[width=0.48\textwidth]{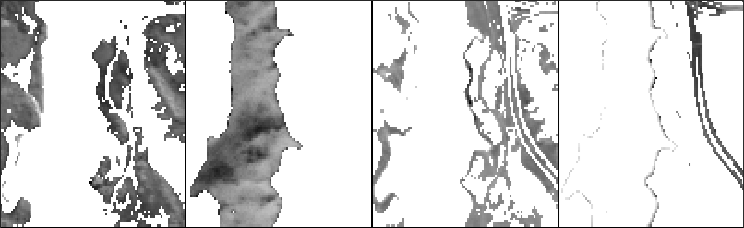}  \\ 
  Ground truth \vspace{0.2cm}    \\ 
 \includegraphics[width=0.48\textwidth]{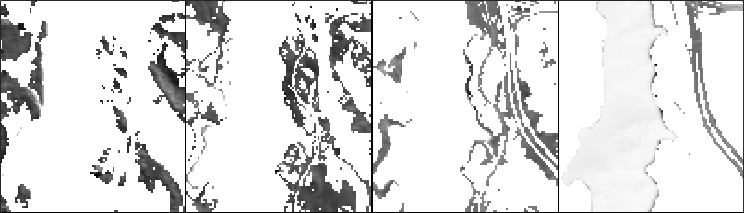}  \\ 
   Fro-NMF  \vspace{0.2cm}  \\
   \includegraphics[width=0.48\textwidth]{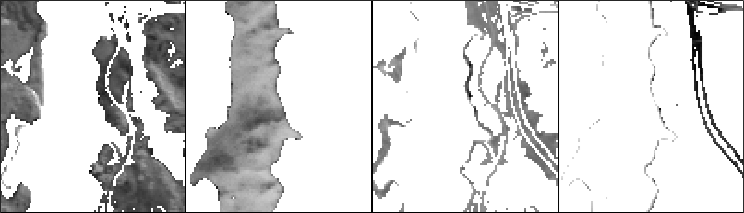}     \\ 
KL-ONMF 
\end{tabular}
\caption{Clustering of the Jasper HSI. From left to right: tree, water, soil and road. \label{fig:spectralsignJasper}} 
\end{center}
\end{figure}


 \section{Conclusion} \label{sec:conclusion}

In this paper, we have proposed a new clustering model for nonnegative data, namely orthogonal NMF with the KL divergence. We designed an alternating optimization algorithm, which is simple but effective and highly scalabe, running in $O(\nnz(X)r)$ operations where  $\nnz(X)$ is the number of non-zero entries in the data matrix, and $r$ is the number of clusters. 
We showed on documents and hyperspectral images that KL-ONMF performs favorably with ONMF with the Frobenius norm, as it provides, on average, better clustering results, while running faster on average.   Further work include the generalization of Algorithm~\ref{alg:klONMF} to any Bregman divergence~\cite{banerjee2005clustering}.

\section*{Acknowledgment} NG acknowledges the support  by the European Union (ERC consolidator, eLinoR, no 101085607).

\bibliographystyle{cas-model2-names}

\bibliography{Article2023}


\end{document}